\documentclass[11pt]{article}
\usepackage{fullpage}
\usepackage{hyperref}       % hyperlinks
\usepackage{url}    % simple URL 
\usepackage{booktabs}       % professional-quality tables
\usepackage{nicefrac}       % compact symbols for 1/2, etc.
\usepackage{microtype}  % microtypography
\usepackage{times}   % use Times
\usepackage{xcolor,colortbl} 
\usepackage{multirow}
\usepackage{tablefootnote}
\usepackage{subcaption}

\usepackage{preamble}
\usepackage{mathnotations}
\usepackage[style=alphabetic,maxnames=12,maxbibnames=10,maxcitenames=10,maxalphanames=10,giveninits=true,doi=false,url=true]{biblatex}
\newcommand*{\citet}[1]{\AtNextCite{\AtEachCitekey{\defcounter{maxnames}{2}}} \textcite{#1}}

\newcommand*{\citep}[1]{\cite{#1}}

\begin{document}
\title{When Determinants Are Not Enough: Private Rare Switching}
\author {
     Xingyu Zhou\thanks{Wayne State University, Detroit, USA.  Email: \texttt{xingyu.zhou@wayne.edu}}
}

\date{}

\maketitle

\begin{abstract}
In this note, I would like to share a small research moment where Codex helped me find the right way to adapt rare switching to the private setting. The standard determinant-based update rule in linear bandits and RL works beautifully because the design matrix grows monotonically. But once Gaussian noise is added for privacy, this monotonicity can fail, and the usual analysis no longer goes through. The key reason is that determinant growth controls volume, while regret analysis needs control of the worst direction. To address this, Codex comes up with a  different rare-switching rule based on the generalized Rayleigh quotient, which restores logarithmic policy updates and the desired confidence-width comparison up to a constant  factor. I present my manually clean-up version of the proof here as well as some personal reflection on this example.
\end{abstract}

% \newpage
% \tableofcontents
% \newpage

\section{The Problem}
The problem is a variant of the standard rare switching policy update in linear bandit and RL. Let us first recall the standard setting, i.e., Sec. 5.1 in~\cite{abbasi2011improved}. The learner maintains a positive-definite covariance matrix $V_t = \lambda I + \sum_{s=1}^t x_s x_s^{\top}$ for $t \in [T], \lambda >0$ and only updates its policy when $\det(V_t) > \alpha \det(V_{\tau})$ for some $\alpha >1$ with $\tau$ being the latest update time slot before $t$. 
This simple rule is often called \emph{rare-switching update} which can (i) reduce the number of total policy updates to $O(\log T)$ (which directly follows from the standard elliptical potential lemma, Lemma 11 in~\cite{abbasi2011improved}) while (ii) only suffers a constant factor blow up in regret (which directly follows from Lemma 12 in~\cite{abbasi2011improved}). In particular, Lemma 12 in~\cite{abbasi2011improved} implies that if \emph{$V_t \succeq V_{\tau}$}, then we have for any $x \in \Real^d$, we have $\norm{x}_{V_{\tau}^{-1}} \le \sqrt{\alpha} \norm{x}_{V_t^{-1}}$.

Now, what if the covariance matrix is perturbed, say, due to privacy protection. For example, in private linear contextual bandits or linear RL, Gaussian noisy matrix is often added to the non-private to arrive at the private one: $\widetilde{V}_t = V_t + E_t$, where $E_t$ is symmetric matrix with elements sampled from Gaussian. A natural question to ask is: Can we still apply the above update rule (on $\widetilde{V_t}$) along with its analysis to arrive at the same performance, i.e., log-order updates with the same-order of regret?

The answer is no. The reason is that to apply Lemma 12, we have to guarantee the monotonicity property, which is $\widetilde{V}_t \succeq \widetilde{V}_{\tau}$. This is no longer true due to the private noise. This issue has been pointed by my previous work~\cite{chowdhury2022shuffle} (cf. Sec. 6) and also by my most recent work~\cite{he2026towards} (cf. Appendix C). A similar issue but in a different context was also identified in my previous work~\cite{zhou2024differentially} (cf. Appendix C.3).

So, we may ask: \emph{how about a new rare-switching rule for the private case}\footnote{Of course, one may ask if we can keep the same update rule but with a more advanced analysis to maintain the same performance guarantees. It turns out that this is not possible, which will become clear at the end of this post.}?

\section{The New Rule}

I basically present the above write-up as a prompt to Codex 5.5 (Extra High) with all related papers in the folder. The first attempt of Codex gives me the correct answer, which is summarized in the following theorem\footnote{I clean up and re-organize the original proof of Codex, which is correct but not that easy to read.}. 

\begin{theorem}
    Let $x_1,\ldots, x_T \in \Real^d$ satisfy $\norm{x_t}_2 \le L$ for all $t \in [T]$. Let the non-private design matrices be
    \begin{align*}
        V_1 := \lambda I, V_{t+1}:= V_t + x_t x_t^{\top}.
    \end{align*}
    Assume the private matrices be $\widetilde{V}_t = V_t  + E_t$ where $E_t$ is a symmetric matrix such that for all $t \in [T]$\footnote{For Gaussian matrix, this condition holds with high probability}
    \begin{align*}
        \norm{E_t}_2 \le \eta \text{ and } 0\le 2 \eta  \le \lambda.
    \end{align*}

    Suppose that the update rule is that for any $t \in [T]$, it updates the policy when 
    \begin{align*}
        \lambda_{\max}\left(\widetilde{V}_{\tau_t}^{-1/2} \widetilde{V}_t \widetilde{V}_{\tau_t}^{-1/2}\right) > \alpha,
    \end{align*}
    where $\tau_t$ is again the most recent update before $t$.

    Then, let $\rho:= \eta /\lambda$ and $c_{\rho}:= \frac{1+\rho}{1-\rho}$, we have 
    \begin{enumerate}[(i)]
        \item For any $x \in \Real^d$ and $t \in [T]$, $\norm{x}_{\widetilde{V}_{\tau_t}^{-1}} \le \sqrt{\alpha} \norm{x}_{\widetilde{V}_t^{-1}}$ and $\norm{x}_{V_{\tau_t}^{-1}} \le \sqrt{c_{\rho} \alpha} \norm{x}_{V_t^{-1}}$;
        \item The total number of updates $m$ is upper bounded by $m \lesssim d \log_{\alpha/c_{\rho}} \left(1 + \frac{T L^2}{\lambda d}\right) = O_d(\log T)$.
    \end{enumerate}
\end{theorem}
\begin{remark}
Several remarks are in order. First, 
    in most cases, $\lambda$ is set to be $2\eta$ so that the private design matrix is also positive definite, see~the proof of Lemma A.4 in \cite{chowdhury2022shuffle}. Hence, $\rho = 1/2$ and $c_{\rho} = 3$. Now, with the above update rule and guarantees, we have the right way to do rare switching in private linear bandit and RL. Second, note that we often only needs the first inequality in (i) to bound the regret. Finally, we also note that the new update rule is also equivalent to that the learner updates iff $\widetilde{V}_t \npreceq  \alpha \widetilde{V}_{\tau_t}$, which is more expensive in computation compared with the old determinant-based update rule (which can leverage the efficient rank-one update).
\end{remark}
\begin{proof}
    We first present the following claim, which will be useful for the proof with the proof being given at the end.
    \begin{claim}
    \label{clm}
        For $ 1\le a \le b \le T$, let $g_{a,b}:= \lambda_{\max}\left({V}_{a}^{-1/2} {V}_b {V}_{a}^{-1/2}\right)$ and $\widetilde{g}_{a,b}:= \lambda_{\max}\left(\widetilde{V}_{a}^{-1/2} \widetilde{V}_b \widetilde{V}_{a}^{-1/2}\right)$. Then, we have $ c_{\rho}^{-1} g_{a,b} \le \widetilde{g}_{a,b} \le c_{\rho} g_{a,b}.$
    \end{claim}

    We start with (i). For any $\tau_t < t$, i.e., $t$ is not the update slot, the first inequality follows directly from the update rule. For the second one, by Claim~\ref{clm} and our update rule, we first have 
    \begin{align*}
         \lambda_{\max}\left({V}_{\tau_t}^{-1/2} {V}_t {V}_{\tau_t}^{-1/2}\right) \le c_{\rho} \alpha. 
    \end{align*}
    Further, by the identity of generalized Rayleigh quotient, we have 
    \begin{align*}
        \sup_{x \neq 0} \frac{x^{\top} V_t x}{x^{\top} V_{\tau_t} x} = \lambda_{\max}\left({V}_{\tau_t}^{-1/2} {V}_t {V}_{\tau_t}^{-1/2}\right) \le c_{\rho} \alpha.
    \end{align*}
    This implies that $V_{\tau_t}^{-1} \preceq c_{\rho} \alpha V_t^{-1}$, and hence $\norm{x}_{V_{\tau_t}^{-1}} \le \sqrt{c_{\rho} \alpha} \norm{x}_{V_t^{-1}}$.

    We now move to (ii). Consider two adjacent policy update slots $i < j$. By the update rule and Claim~\ref{clm}, we have 
    \begin{align*}
      \lambda_{\max}\left({V}_{i}^{-1/2} {V}_j {V}_{i}^{-1/2}\right) >  \alpha /c_{\rho}. 
    \end{align*}
    Further, since $V_j \succeq V_i$, all eigenvalues of ${V}_{i}^{-1/2} {V}_j {V}_{i}^{-1/2}$ are least one. Hence, by the above result and multiplicativity of determinant, we have 
  \begin{align*}
      \frac{\det(V_j)}{\det(V_i)} = \det\left({V}_{i}^{-1/2} {V}_j {V}_{i}^{-1/2}\right) > \alpha /c_{\rho}.
  \end{align*}
  Thus, by the elliptical potential lemma (cf. Lemma 11 in~\cite{abbasi2011improved}), we have that the total number of updates $m$ is upper bounded by $m \lesssim d \log_{\alpha/c_{\rho}} \left(1 + \frac{T L^2}{\lambda d}\right)$.

  We are left with the proof of Claim~\ref{clm}. To show it, we first notice that by the condition on $E_t$, we have 
  \begin{align*}
      -\eta I \preceq E_t \preceq \eta I.
  \end{align*}
  Further, by $V_t \succeq \lambda I$, hence  $\eta I \preceq (\eta/\lambda) V_t = \rho V_t$. Combing the above, yields that 
  \begin{align*}
     (1-\rho) V_t \preceq \widetilde{V_t} \preceq (1+\rho) V_t.
  \end{align*}

  Thus, for any nonzero vector $x \in \Real^d$, we have 
  \begin{align*}
      \frac{x^{\top} \widetilde{V}_b x}{x^{\top} \widetilde{V}_a x} \le \frac{(1+\rho) x^{\top} {V}_b x}{(1-\rho) x^{\top} {V}_a x} = c_{\rho} \frac{ x^{\top} {V}_b x}{x^{\top} {V}_a x}.
  \end{align*}
  Taking sup and applying identity of generalized Rayleigh quotient, yields that $\widetilde{g}_{a,b} \le c_{\rho} g_{a,b}$. The other direction follows the same proof. 
\end{proof}

\section{The Connection}
One may wonder what's the difference or connection between the new update rule and the standard one. To see this, it is instructive to see the behavior of the new update rule even in the non-private case and ask whether it leads to a more frequent or less frequent update. To this end, I directly prompt Codex (also with ChatGPT) with the the screenshot of Lemma 12 in~\cite{abbasi2011improved}. It then gives me a new proof which actually sheds a lot of light on how Codex comes up with the new update rule and its connection with the old one.

The new proof is much simpler (at least to me) compared with the original proof. Basically, the proof first notes that LHS in Lemma 12 is equal to 
\begin{align}
\label{eq:new}
    \sup_{x \neq 0} \frac{x^{\top}A x}{x^{\top} B x} =  \lambda_{\max}(B^{-1/2} A B^{-1/2}),
\end{align}
where again we apply identity of generalized Rayleigh quotient, while the RHS in Lemma 12 is equal to 
\begin{align}
\label{eq:old}
    \frac{\det(A)}{\det(B)} = \det(B^{-1/2} A B^{-1/2}).
\end{align}
Note that~\eqref{eq:new} $\le$~\eqref{eq:old} by the fact that all eigenvalues of  $B^{-1/2} A B^{-1/2}$ are at least one, which comes from the monotonicity that $B \preceq A$. This proves Lemma 12.

So, essentially, the new update rule comes from the new proof of Lemma 12, i.e.,~\eqref{eq:new}. Moreover, we can see that under the new update rule, the learner would update less frequent as $\eqref{eq:new} > \alpha$ implies $\eqref{eq:old} > \alpha$, but not the other direction.

This new proof also explains why the new update rule can get rid of the issue in the old rule when lacking of monotonicity. This is because using the old update rule (cf.~\eqref{eq:old}), one has to further leverage the monotonicity to arrive at the bound on $ \sup_{x \neq 0} \frac{x^{\top}A x}{x^{\top} B x}$, i.e., the worst-case direction. On the other hand, the new rule directly works with $ \sup_{x \neq 0} \frac{x^{\top}A x}{x^{\top} B x}$.

We can now also see that with the old update rule, it is impossible to control the worst case direction without monotonicity since the eigenvalues on other directions can be much smaller and hence a very large value on the worst-case direction can still make the determinant small.

\section{The Reflection}
% Codex offers a new rare-switching rule in the private case that preserves the nice properties. After a closer inspection, this new rule can be traced to an alternative proof of a key lemma. The techniques behind the proof are definitely not new. However, as stated in Sebastien Bubeck 's post \emph{``yet at the same time it is also true that the model didn’t “invent” any “new mathematics” (say it didn’t invent some alternative class field theory, whatever that would mean).  But this is the crucial point: merely being able to know deeply all the results in a scientific field, and being able to use all known arguments expertly and with just the right choice of parameters, that alone can lead to a ton of breakthroughs, and this is not just limited to mathematics''}, I think the above example shares the same spirit. 

Codex suggests a new rare-switching rule for the private case that preserves the desired guarantees. Upon closer inspection, the rule can be traced back to an alternative proof of a key lemma. The ingredients in the proof are not new, but the useful part is the way they are selected, combined, and applied to the right object. In this sense, this mini example somewhat the same spirit as the unit distance problem.

This reminds me of several things. First, for me, the most enjoyable part of doing research is not necessarily its impact, nor the final breakthrough result. Rather, it is the moment when I understand something a little more deeply\footnote{Of course, this is a privilege thing. Moreover, it can also lead to an embarrassing outcome where there is neither impact nor deep understanding. You can name me here.}. In this sense, research carries its own intrinsic reward. One thing I still remember from high school and college is that I often went to my teachers or professors to ask clarification questions, or to explain my own understanding and hope they could verify it.\footnote{Using today's fancy language, I was trying to build and calibrate my own world model.} Now, with AI, students may have a much more convenient way to test, refine, and verify their own understanding (if they want...).

Second, I tend to view deep understanding of a subject as learning a good representation. A good representation compresses many details into the right concepts, and because of that, it can be flexible and useful when adapted to new problems. 
In many cases, what is needed to make a significant impact, either in theory or in practice, is not necessarily a completely new theory or method. Rather, it may be something that grows naturally out of a better compression or representation of existing knowledge, one that identifies the ``core set'' of ideas from which a much larger unknown space can be explored.

So, with today's AI tools, humans may become more efficient at maximizing this kind of intrinsic reward: understanding more, clarifying faster, and building better representations. Perhaps we should also think about how to steer AI systems, either during training or through test-time scaling, toward a similar objective: not merely producing answers, but actively seeking good representations, useful abstractions, and core sets of ideas. To me, this connects naturally to the fundamental problem of ``exploration'' in reinforcement learning, which is currently my biggest interest.

\newpage
\printbibliography

@inproceedings{zhou2024differentially,
  title={On differentially private federated linear contextual bandits},
  author={Zhou, Xingyu and Ray Chowdhury, Sayak},
  booktitle={International Conference on Learning Representations},
  volume={2024},
  pages={30101--30131},
  year={2024}
}

@article{he2026towards,
  title={Towards Differentially Private Reinforcement Learning with General Function Approximation},
  author={He, Yi and Zhou, Xingyu},
  journal={arXiv preprint arXiv:2605.07049},
  year={2026}
}

@inproceedings{chowdhury2022shuffle,
  title={Shuffle Private Linear Contextual Bandits},
  author={Chowdhury, Sayak Ray and Zhou, Xingyu},
  booktitle={International Conference on Machine Learning},
  pages={3984--4009},
  year={2022},
  organization={PMLR}
}

@article{abbasi2011improved,
  title={Improved algorithms for linear stochastic bandits},
  author={Abbasi-Yadkori, Yasin and P{\'a}l, D{\'a}vid and Szepesv{\'a}ri, Csaba},
  journal={Advances in neural information processing systems},
  volume={24},
  year={2011}
}
\newpage

\appendix
\end{document}